\documentclass[letterpaper]{article} 
\usepackage{aaai24}  
\usepackage{times}  
\usepackage{helvet}  
\usepackage{courier}  
\usepackage[hyphens]{url}  
\usepackage{graphicx} 
\usepackage{natbib}  
\usepackage{caption} 
\usepackage{algorithm}
\usepackage{algorithmic}
\usepackage{subfigure}
\usepackage{amsfonts}
\usepackage{booktabs} 
\usepackage{bigdelim} 
\usepackage{amsmath}
\usepackage{amssymb}
\usepackage{array}
\usepackage{stackengine}
\usepackage{boldline}
\newtheorem{theorem}{Theorem}
\newtheorem{proof}{Proof}
\usepackage{array}
\usepackage{booktabs} 

\usepackage{newfloat}
\usepackage{listings}
\DeclareCaptionStyle{ruled}{labelfont=normalfont,labelsep=colon,strut=off} 
\floatstyle{ruled}
\newfloat{listing}{tb}{lst}{}
\floatname{listing}{Listing}

\nocopyright

\title{
SieveNet: Selecting Point-Based Features for Mesh Networks
}
\author{
    Shengchao Yuan\textsuperscript{\rm 1}, 
    Yishun Dou\textsuperscript{\rm 2},
    Rui Shi\textsuperscript{\rm 1},
    Bingbing Ni\textsuperscript{\rm 1}\thanks{Bingbing Ni is the corresponding author.},
    Zhong Zheng\textsuperscript{\rm 2}\\
}
\affiliations{
    \textsuperscript{\rm 1}Shanghai Jiao Tong University, Shanghai 200240, China \quad
    \textsuperscript{\rm 2}Huawei\\
    \{sc\_yuan, shi-rui, nibingbing\}@sjtu.edu.cn, 
    douyishun@hisilicon.com,
    zhengzhong5@huawei.com
}

\usepackage{bibentry}

\begin{document}

\maketitle

\begin{abstract}
    Meshes are widely used in 3D computer vision and graphics, but their irregular topology poses challenges in applying them to existing neural network architectures. Recent advances in mesh neural networks turn to remeshing and push the boundary of pioneer methods that solely take the raw meshes as input. 
    Although the remeshing offers a regular topology that significantly facilitates the design of mesh network architectures, features extracted from such remeshed proxies may struggle to retain the underlying geometry faithfully, limiting the subsequent neural network's capacity. To address this issue, we propose SieveNet, a novel paradigm that takes into account both the regular topology and the exact geometry. Specifically, this method utilizes structured mesh topology from remeshing and accurate geometric information from distortion-aware point sampling on the surface of the original mesh. Furthermore, our method eliminates the need for hand-crafted feature engineering and can leverage off-the-shelf network architectures such as the vision transformer. Comprehensive experimental results on classification and segmentation tasks well demonstrate the effectiveness and superiority of our method.
  \end{abstract}

\section{Introduction}
Polygonal mesh is one of the most commonly used data structures in computer vision and graphics to represent 3D objects, which concisely describes the geometry and topology of the object through vertices, edges, and faces. It is a popular approach in various applications, such as modeling, rendering, and animation.
In contrast to other prevalent forms of 3D representations, meshes offer a concise and all-encompassing depiction of three-dimensional objects. As opposed to voxels, which capture volumetric data, meshes exclusively delineate the surface geometry of an object, thereby furnishing a more compact representation. Compared with point clouds, meshes allow more accurate and detailed modeling of complex shapes and features by defining the topology.

However, the flexible nature of the mesh representation results in the mesh data structure being unstable and irregular. The topology of a mesh is unstable, even for 3D objects with identical or similar geometries. For example, a rectangle with four vertices can be partitioned into two distinct sets of triangular faces by connecting different diagonals. Furthermore, the topology of meshes is irregular. The connectivity of the units in the meshes, such as vertices, edges, and faces, needs to be described by complex graphs, which are distinct from the situation in 2D pixel arrays. 
These inherent properties of meshes make it challenging to take advantage of mature network architectures designed for the image domain.

\begin{figure}[t]
  \centering
  \vspace{-20pt}
  \includegraphics[width=1.0\linewidth]{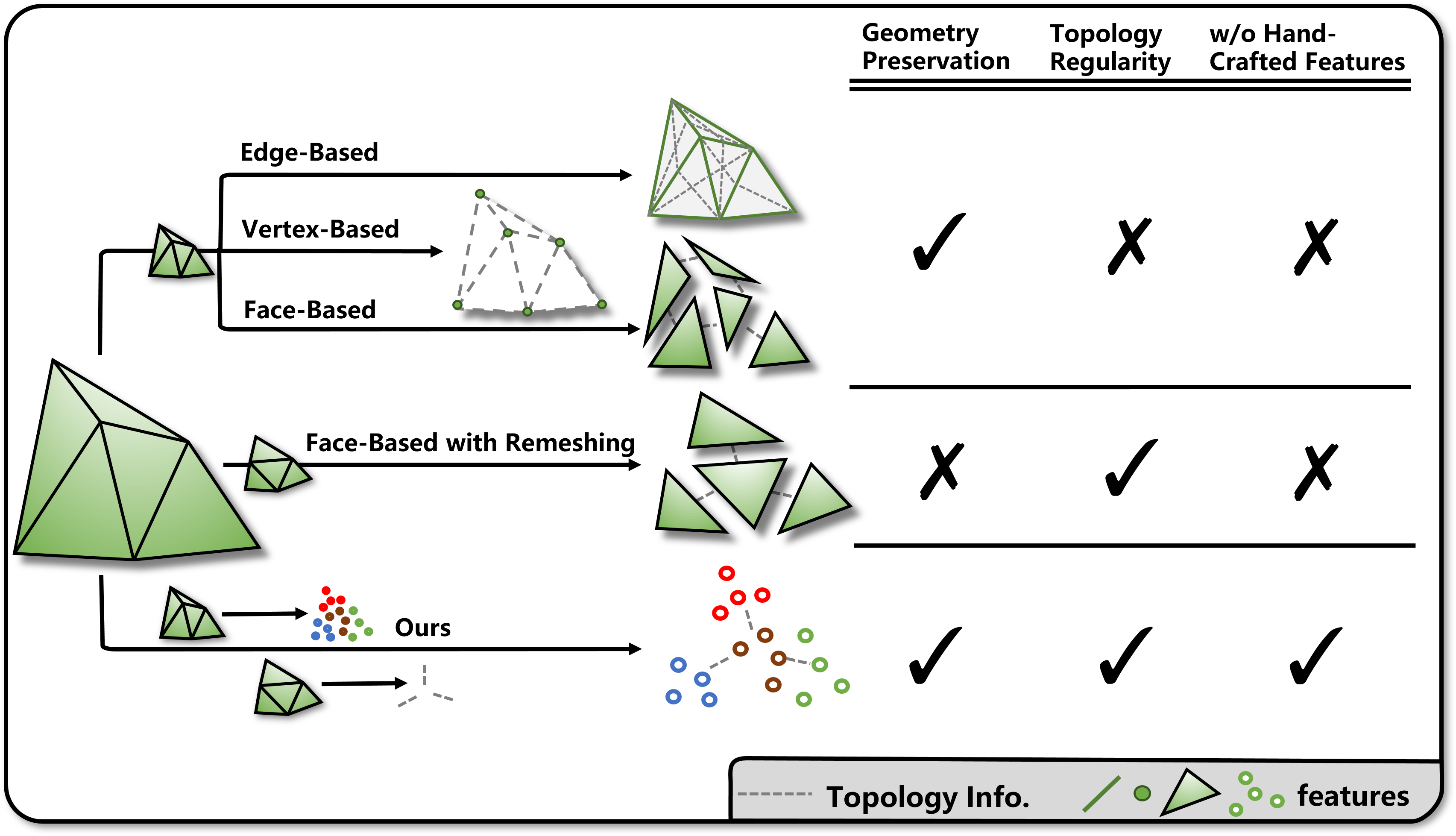}
  \caption{\label{fig:compare}Deep learning methods on meshes. Methods that directly take raw meshes as inputs (rows 1-3) are hampered by the irregular topology inherent in such meshes. Meanwhile, some remeshing methods (row 4) rely on geometrically imprecise proxies. To bypass these limitations, we unite the precise geometry from raw meshes with the regular topology afforded by remeshing. Moreover, by using point-based features described only by positions and normals, our method eliminates the need for hand-crafted descriptors.
  }
\end{figure}

In an effort to solve this problem, attempts have been made to extend the paradigms in 2D images to 3D meshes. Several pioneering methods have defined operations directly on unmodified mesh data. The classic MeshCNN~\cite{hanocka19meshcnn} defines convolution and pooling operations on the mesh edges and designs input descriptors related to angles and edge lengths on the half-edge structure. PD-MeshNet~\cite{milano20pdmeshnet} extends a primal-dual framework from the graph neural network to triangle meshes to define convolutions and introduces a pooling operation by dynamically collapsing graph edges using attention coefficients. MeshNet~\cite{feng19meshnet} defines operations on faces, noticing that faces are rich in spacial and structural features and have a regular connection relationship in two-manifold triangle meshes. Recently, methods based on remeshing have gradually gained attention. 
SubdivNet~\cite{hu2022subdivnet} finds that the subdivision surface has a topology similar to 2D pixels, and all 2-manifold triangular meshes can be remeshed into subdivision surfaces. Therefore, mesh operators can be designed by simply analogizing mesh patches to pixels.
MeshMAE~\cite{liang22meshmae} introduces SubdivNet's~\cite{hu2022subdivnet} scheme into the transformer model and pre-trains the network using the masked autoencoder~\cite{he22mae} technique.


\begin{figure}[t]
  \centering
  \hfill
  \subfigure[]{
        \includegraphics[width=0.45\linewidth]{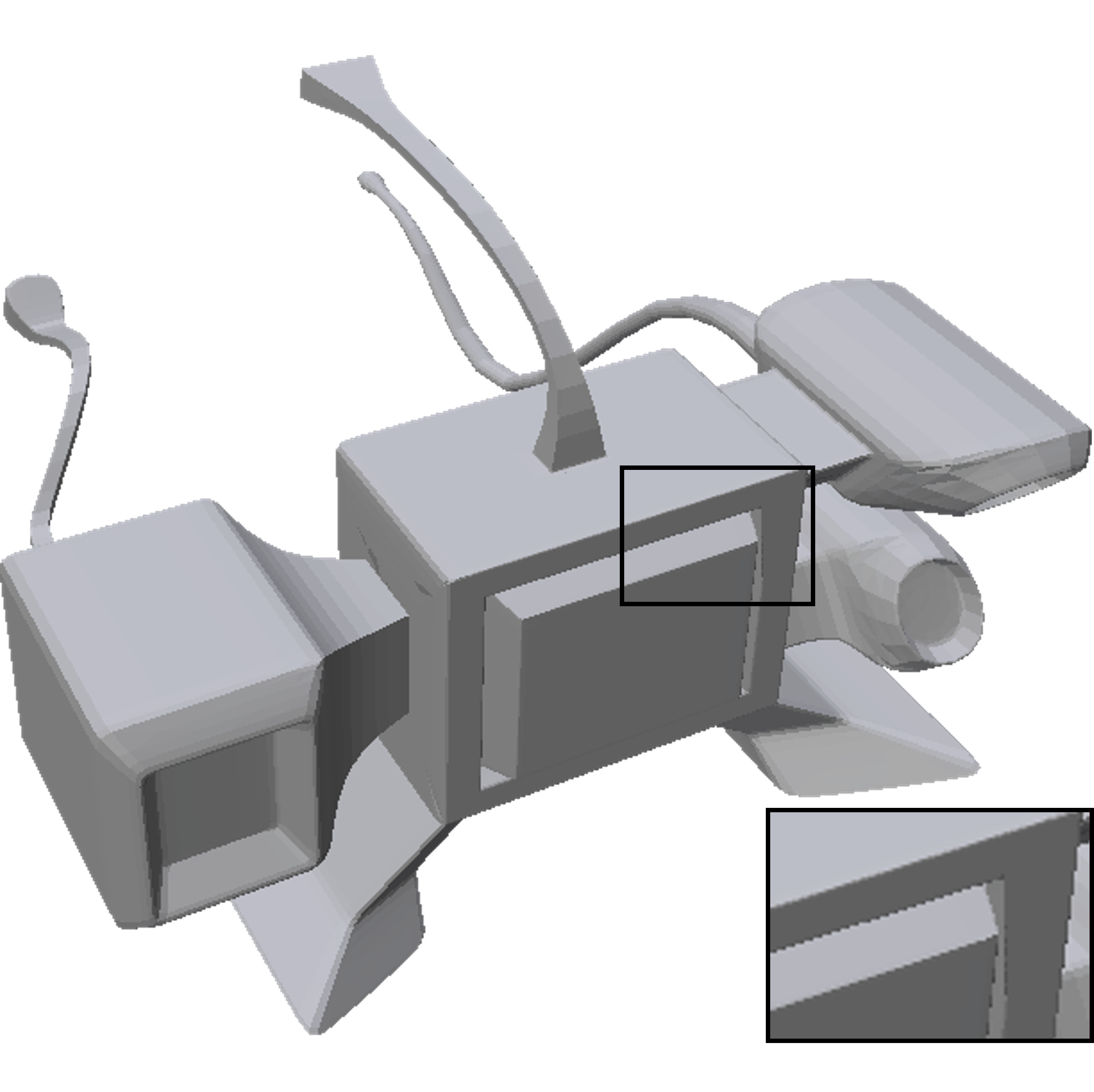}
        \label{fig:raw}
    }
  \hfill
  \subfigure[]{
        \includegraphics[width=0.45\linewidth]{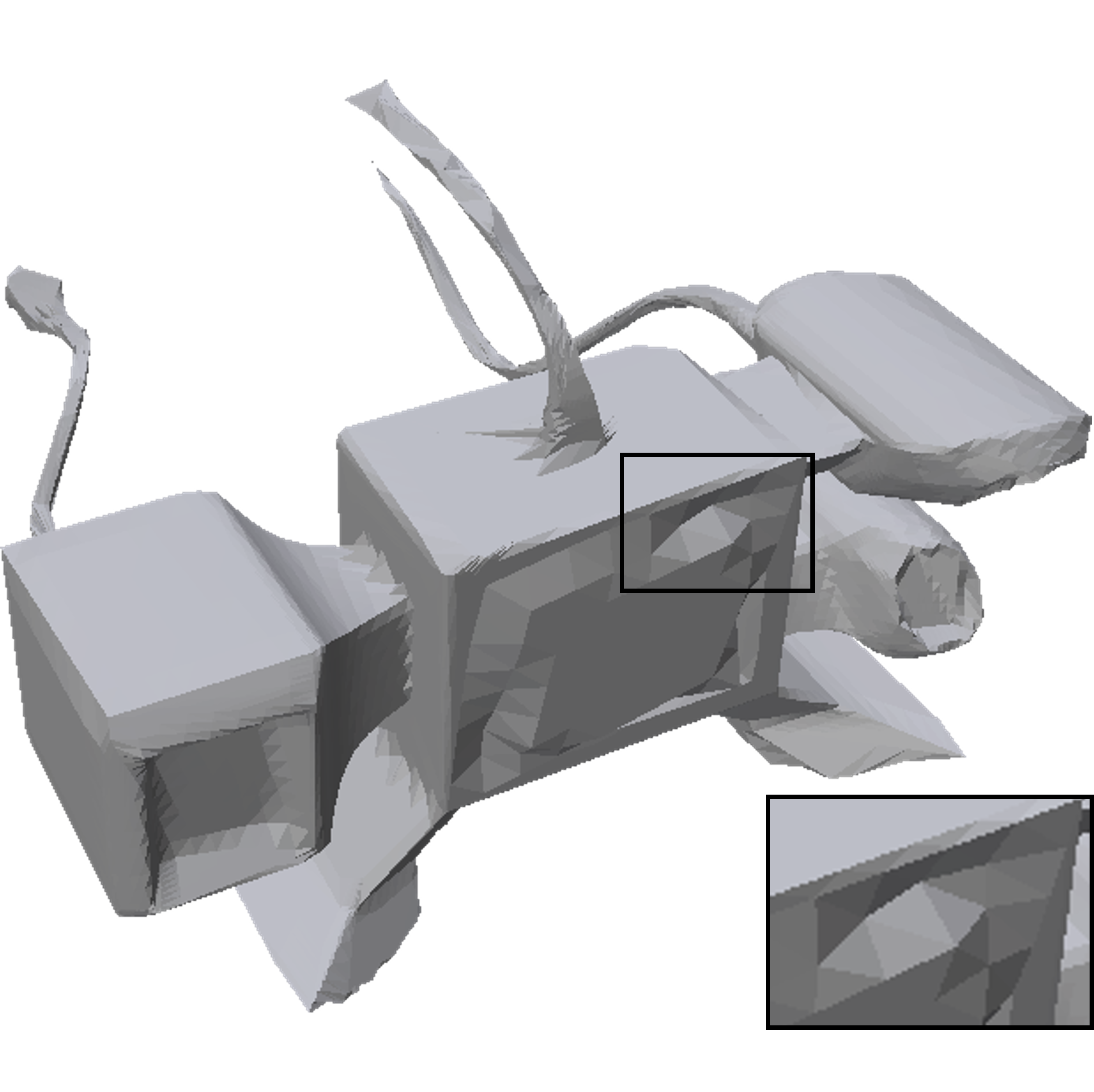}
        \label{fig:remesh}
      }
  \hfill  \caption{\label{fig:remeshdegrad}Remeshing leads to a perceivable degradation of geometric quality. (a) is the original mesh, and (b) is the remeshing result of MAPS~\cite{lee98maps}, the method used in state-of-the-art networks such as SubdivNet~\cite{hu2022subdivnet} and MeshMAE~\cite{liang22meshmae}. The example is from the COSEG-alien dataset~\cite{wang12coseg}.}
\end{figure}

These remesh-based methods introduce an organized structure to the mesh topology, while we notice that the remeshing process leads to perceptible degradation in mesh geometric quality, as shown in Fig.~\ref{fig:remeshdegrad}. A proxy mesh with such imprecision may not faithfully capture the geometric information on the original mesh, thus limiting the capacity of subsequent neural networks. 

To tackle this problem, we propose SieveNet, a novel paradigm that considers both the regular topology from remeshing and the exact geometry from the original mesh surface, as shown in Fig.~\ref{fig:compare}.
Given a 2-manifold triangle mesh, we first extract a regular topology by mesh simplification and mesh subdivision, where each face in the subdivided mesh only serves as a topology unit to pack the features from the original mesh. The regular topology units allow us to reuse the off-the-shelf models as remesh-based methods.
In the meantime, we trace the process to establish a point-to-point bijection between the original mesh and the subdivided mesh. With the help of the bijection, the point-based geometric features from the original mesh are projected and packed into the regular topology units. 
For this purpose, we design a scheme that tries to uniformly sample features from the original mesh while keeping the features in the topology units with the same size. This stage starts with stratified sampling candidate points on the subdivided mesh, followed by a distortion-aware selection to approximate the uniformity of the original mesh.
The point-based features can be fully described by positions and normals of the point, eliminating the need for hand-crafted face features (such as the face area, inner angles, curvatures, and the face center~\cite{feng19meshnet}).

In summary, our work makes the following contributions:
\begin{itemize}
\item We propose SieveNet, a novel scheme to combine the regular topology from the subdivided mesh and the faithful geometric features from the original mesh.
\item The scheme eliminates the need for hand-crafted features and can leverage off-the-shelf network architectures.
\item Experiments on several classification and segmentation datasets demonstrate the effectiveness of our method.
\end{itemize}

\section{Related Work}
\subsection{3D Geometry Networks}
In recent years, various 3D deep learning approaches have emerged. Voxel methods~\cite{Zhirong15CVPR, Maturana2015VoxNet} extend 2D CNNs to 3D by encoding geometry into dense 3D grids. However, their high memory overhead makes sparse 3D data representation inefficient. Subsequent works like octrees~\cite{wang17ocnn} explore hierarchical structures  to improve scalability. View-based methods~\cite{Su15mvcnn, sinha16geoim, feng18gvcnn} project 3D objects into 2D images to leverage image networks. However, they discard spatial 3D relationships and are sensitive to object pose.


\smallskip\noindent\textbf{Point Cloud-Based Methods.} Point cloud methods represent surfaces via unstructured point sets, where data points can be easily collected by sensors like LiDAR or synthesized from meshes. Architectures like PointNet~\cite{qi17pointnet} and PointNet++~\cite{qi17pointnetv2} enable deep CNNs to process point clouds in their unstructured format, and subsequent works focus on mechanisms to extract local structures~\cite{9010002, zhao2021point, 9577737}. Recently, PointNeXt~\cite{qian2022pointnext} introduces residual connections and separable MLPs into PointNet++, enabling  effective and efficient model scaling. Despite these advances, the lack of topological information in point clouds means that the inference of local structures relies on extra clustering methods.

\smallskip\noindent\textbf{Mesh-Based Methods.} Mesh-based approaches address these limitations and appear promising. These methods can be categorized into vertex-based~\cite{masci15gcnn, boscaini16acnn, sharp22diffnet}, edge-based~\cite{milano20pdmeshnet}, and face-based~\cite{feng19meshnet, liu20neusubdiv, liang22meshmae} approaches. MeshCNN~\cite{hanocka19meshcnn} operates on the edge of the raw meshes by analogizing convolution and pooling on the mesh topology.
SubdivNet~\cite{hu2022subdivnet} proposes a promising approach that
introduces remeshing for face-based methods, converting raw meshes to similar meshes with regular topology, simplifying the design of mesh networks. 
MeshMAE~\cite{liang22meshmae} follows the remeshing scheme and utilizes the masked autoencoder to pre-train a transformer.


\begin{figure*}[t]
    \centering
    \includegraphics[width=0.98\linewidth]{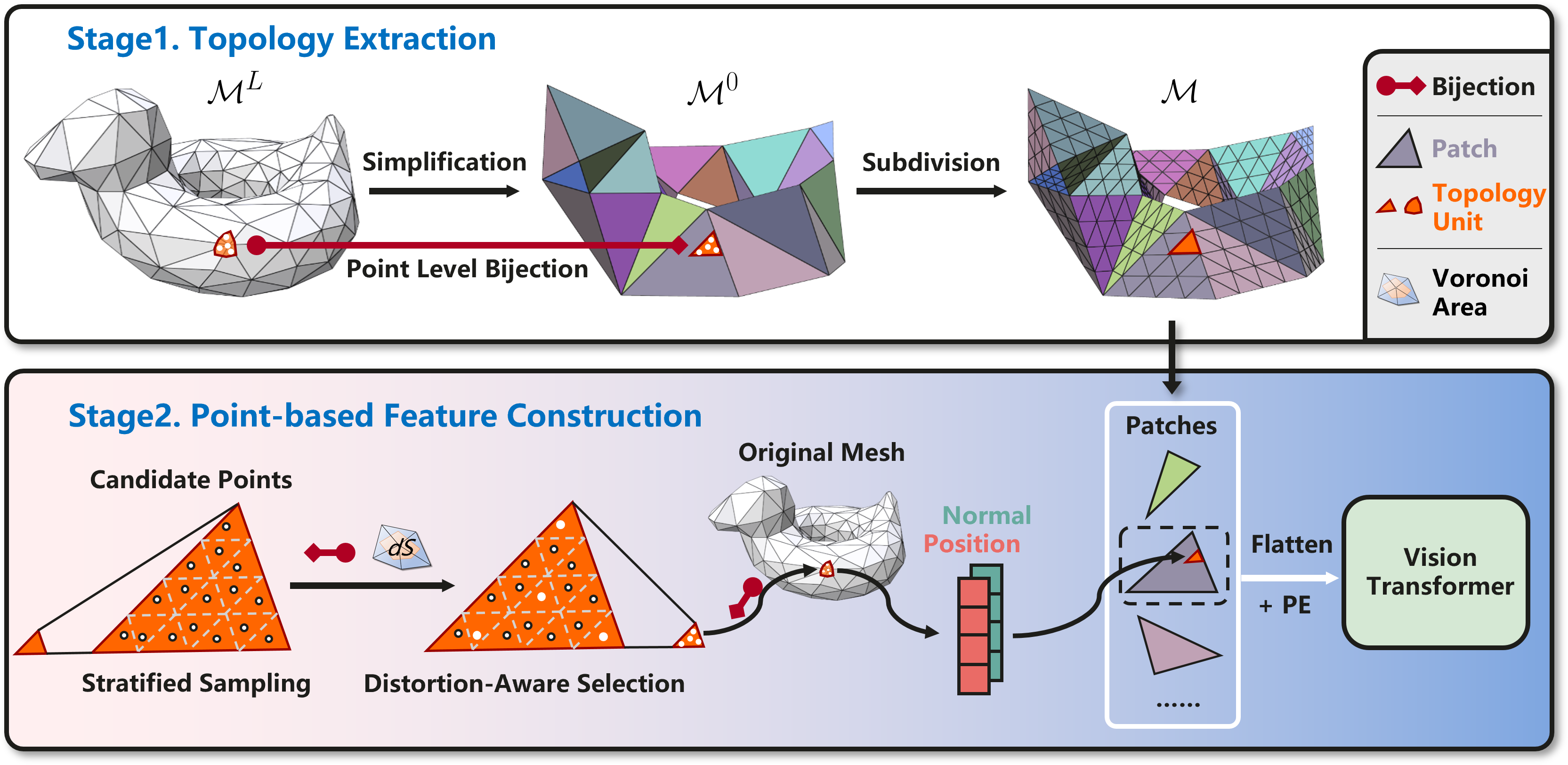}
    \caption{
    Illustration of our pipeline.
    Our method takes into account both the regular topology and 
    precise raw geometry. \quad
    Stage 1. Given a triangle mesh, simplification and subdivision 
    are conducted to extract a regular and fine-grained topology. 
    In the meantime, a point-level bijection is established. 
    Stage 2. 
    We start with stratified sampling candidate points on the subdivided mesh $\mathcal{M}$, 
    followed by a distortion-aware selection such that the selected points 
    seem \textit{uniformly} sampled from the original mesh $\mathcal{M}^L$ within each topology unit. 
    The positions and normals of selected points are packed together to represent a 
    topology unit, which is then further packed and flattened to form a patch representation. 
    With an off-the-shelf vision transformer, our method can achieve state-of-the-art results.}
    \label{fig:pipeline}
  \end{figure*}

\subsection{Simplification and Subdivision}
\noindent\textbf{Simplification.} 
Mesh simplification aims to reduce mesh faces while preserving visual appearance. Traditional methods include vertex decimation~\cite{rossignac1993multi, schroeder1992decimation} and edge collapse~\cite{garland97qem, 10.1145/3209661}, each minimizing heuristic cost or energy. Recent learnable techniques~\cite{9878835} claim real-time mesh simplification. Our SieveNet employs the classic edge collapse method QEM~\cite{garland97qem} together with a subsequent subdivision to build a regular topology.

\smallskip\noindent\textbf{Subdivision.} Subdivision refines coarse meshes into smooth surfaces through recursive refinements. 
Classic schemes include interpolating~\cite{catmull1978recursively, doo1978subdivision, loop87loopsubdiv, zorin1996interpolating, kobbelt2000sqrt} and approximating~\cite{dyn90butterfly, kobbelt96quad}.
Neural network methods~\cite{10.1145/3506694, 10.1145/3386569.3392418, 10.1145/3588432.3591531} enhance subdivision schemes for richer geometry.
Notably, since the remeshed proxy in our approach only provides a regular topology, we ignore vertex updates and only subdivide triangles in a 1-to-4 manner like Loop~\shortcite{loop87loopsubdiv}.

\section{Methodology}
  Motivated by the limitations in remeshing, we aim to design a neural mesh network which can benefit from a regular topology and take the original geometric features as inputs. Figure~\ref{fig:pipeline} shows a visual overview of our method. Given a triangle mesh, we first extract a structured topology through mesh simplification and mesh subdivision. 
  Meanwhile, a point-to-point bijection between the original mesh 
  and the subdivided mesh is established (Sec.~\ref{sec:32TopExtr}). 
  After that, we prepare point-based features by sampling points on each topology unit while extracting point features on the original mesh surface, where the point features are packed together to represent a topology unit, which is further packed to form a patch representation (Sec.~\ref{sec:33GeoSamp}).
  
  Utilizing the structured topology instead of an irregular one in the original mesh simplifies the data structure so that we can organize mesh data in a simple way analogous to 2D images. This allows us to use off-the-shelf network architectures, detailed in Sec.~\ref{sec:network_arch}. Moreover, the geometric features used in our network are point-based, which can be completely described by position and normal, avoiding the need for hand-crafted face features. 
  
  
  \subsection{Topology Extraction}
  \label{sec:32TopExtr}
  \noindent\textbf{Mesh Simplification.}
  Denoting the input mesh as $\mathcal{M}^L$, which is iteratively simplified to 
  $\mathcal{M}^\ell$ $(\ell=L-1, L-2,...,0)$ 
  using $f_{\ell, \ell-1}: \mathcal{M}^{\ell}\rightarrow \mathcal{M}^{\ell-1}$. 
  After $L$ iterations, we obtain the coarsest mesh $\mathcal{M}^0$: 
   \begin{equation}
       \mathcal{M}^0 = (f_{1, 0} \circ\cdots\circ f_{L, L-1})(\mathcal{M}^L).
       \label{eq:meshSim}
   \end{equation}
  The implementation of the mesh simplification in our SieveNet follows
  QEM~\cite{garland97qem}, which iteratively collapses an edge to minimize a geometric error. 
  Specifically, in each iteration, an edge $\boldsymbol{e} = (\boldsymbol{v}_1, \boldsymbol{v}_2)$ is removed, where the endpoints $\boldsymbol{v}_1, \boldsymbol{v}_2$ are merged into a new vertex at the optimal position to minimize the local quadratic error, and the remaining edges connecting to 
  $\boldsymbol{v}_1$ and $\boldsymbol{v}_2$ are connected to $\boldsymbol{v}$. 
  
  The coarse triangles of $\mathcal{M}^0$ are considered as patches for vision transformer, each of which comprises a set of sub-triangles with regular topology, achieved by recursive triangle subdivisions. 
  
  \smallskip\noindent\textbf{Subdivision.}
  We use Loop~\shortcite{loop87loopsubdiv} subdivision without vertex update to obtain a mesh $\mathcal{M}$ with finer topology from the simplified mesh $\mathcal{M}^0$. 
  See Fig.~\ref{fig:subdiv}, the Loop scheme uses a 1-to-4 division to split a triangle face recursively. Given the times of subdivisions, any triangle face is subdivided in the same way. Therefore, the topology between the subdivided faces can be directly expressed in the index relationship. We term the split sub-triangle the topology unit, as it serves as the basic feature unit arranged regularly within a patch. 
  \begin{figure}[htb]
    \centering
    \includegraphics[width=0.7\linewidth]{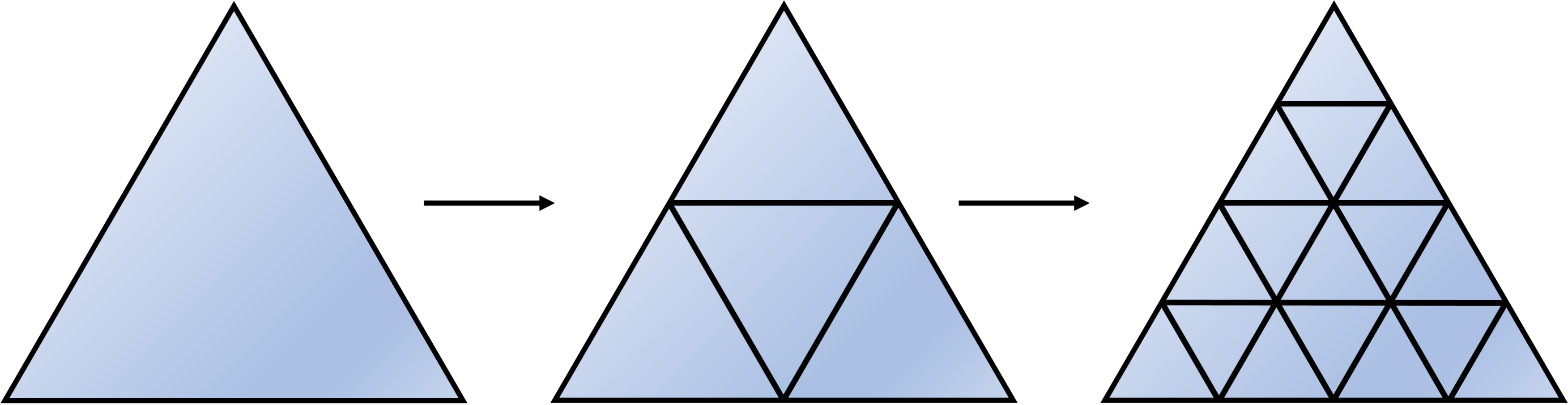}
    \caption{
    Loop subdivision~\cite{loop87loopsubdiv} recursively 
    splits the triangle in a 1-to-4 manner.}
    \label{fig:subdiv}
  \end{figure}
  
  \noindent\textbf{Point Level Bijection.}
  For each iteration of simplification $f_{\ell, \ell-1}$, 
  a bijective mapping $\Psi_{\ell, \ell-1}$ is established between points in $\mathcal{M}^\ell$ and $\mathcal{M}^{\ell-1}$, 
  as illustrated in Fig.~\ref{fig:topextr}. 
  The \textit{bijective} property of $\Psi_{\ell, \ell-1}$ means that 
  $\forall \boldsymbol{p} \in \mathbb{R}^3$ on $\mathcal{M}^\ell$, 
  $\exists \boldsymbol{p}' \in \mathbb{R}^3$ 
  on $\mathcal{M}^{\ell-1}$ such that 
  $\Psi_{\ell, \ell-1}(\boldsymbol{p})=\boldsymbol{p}'$, and vice versa.
  
  Assuming an edge $\boldsymbol{e}=(\boldsymbol{v}_1, \boldsymbol{v}_2)$ is chosen for the collapse. 
  It can be found that the geometry is only affected in the 1-ring region of edge $\boldsymbol{e}$, 
  that is, the region enclosed by the neighbors of the two endpoints $\boldsymbol{v}_1$ and $\boldsymbol{v}_2$. 
  We denote the regions before and after the simplification 
  as $\mathcal{R}^\ell \subset \mathcal{M}^\ell$ and $\mathcal{R}^{\ell-1}\subset \mathcal{M}^{\ell-1}$, 
  respectively. 
  
  \begin{figure}[t]
    \centering
    \includegraphics[width=1.\linewidth]{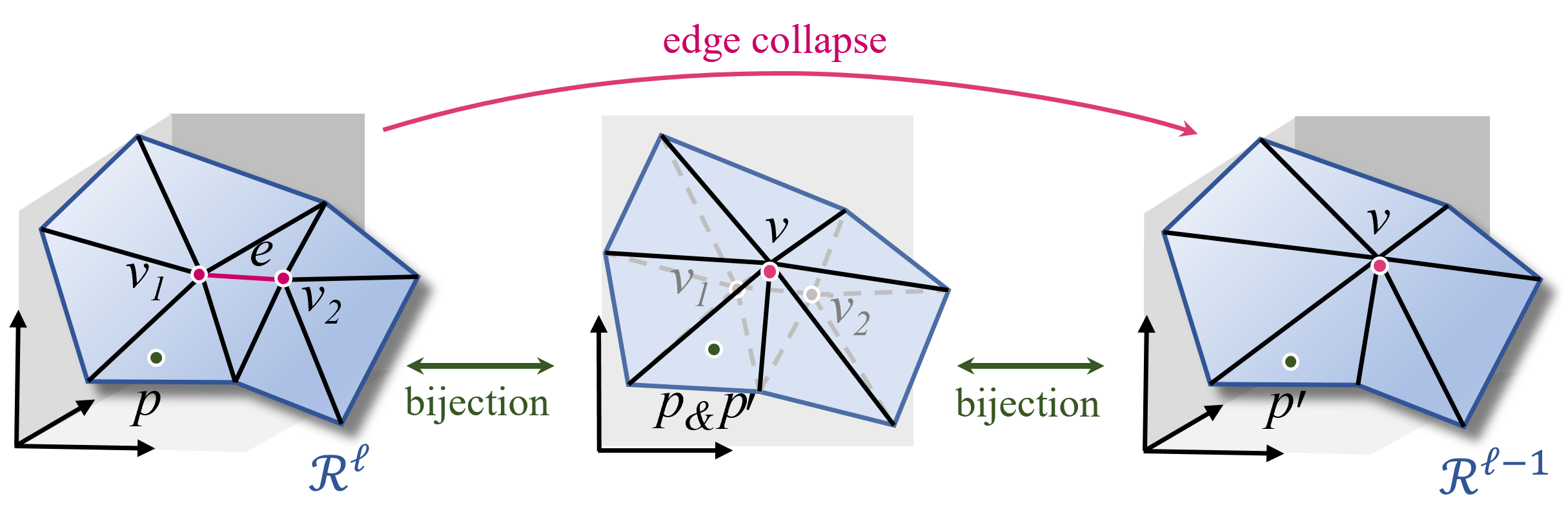}
    \caption{Point level bijection. The 1-ring region $\mathcal{R}^\ell$ (left) of edge $e$ is deformed to $\mathcal{R}^{\ell-1}$ (right) after the edge collapse. 
    By flattening $\mathcal{R}^\ell$ and $\mathcal{R}^{\ell-1}$ into the same 2D UV domain (middle), we can establish a bijection between arbitrary point pair 
    $({p}, {p}')$ in $\mathcal{R}^\ell$ and $\mathcal{R}^{\ell-1}$.}
    \label{fig:topextr}
  \end{figure}
  
  In order to construct a bijection $\Psi_{\ell, \ell-1}$ between $\mathcal{R}^\ell$ and $\mathcal{R}^{\ell-1}$, 
  we first need to find a conformal mapping $\psi$ that maps $\mathcal{R}^\ell$ and $\mathcal{R}^{\ell-1}$ into 
  the same domain, 
  and then construct a bijection between $\mathcal{R}^\ell$ and $\mathcal{R}^{\ell-1}$:
  \begin{align}
    \psi_{\mathcal{R}^\ell}(\boldsymbol{p}) & = \psi_{\mathcal{R}^{\ell-1}}(\boldsymbol{p}'), \\
     \Rightarrow  \ \  \Psi_{\ell, \ell-1}(\boldsymbol{p}) & = \boldsymbol{p}' = \psi_{\mathcal{R}^{\ell-1}}^{-1}(\psi_{\mathcal{R}^\ell}(\boldsymbol{p})),
    \label{eq:bijection}
  \end{align}
  where $\boldsymbol{p}$ is a point in $\mathcal{R}^\ell$ 
  and $\boldsymbol{p}'$ is the corresponding point in $\mathcal{R}^{\ell-1}$. 
  Following Liu et al.~\shortcite{liu20neusubdiv}, 
  the conformal mapping $\psi$ flattens $\mathcal{R}^\ell$ and $\mathcal{R}^{\ell-1}$ 
  into the same UV domain in $\mathbb{R}^2$. 
  For the rest of the meshes $\mathcal{M}^{\ell} \backslash \mathcal{R}^\ell$ 
  and $\mathcal{M}^{\ell-1} \backslash \mathcal{R}^{\ell-1}$, 
  the mapping 
  is identical 
  since they are not modified during the simplification $f_{l, l-1}$ 
  \footnote{$A \backslash B$ stands for subtracting set $B$ from set $A$.}.
  
  
  By iteratively applying the bijections on the mesh sequence 
  $\mathcal{M}^L, \cdots, \mathcal{M}^0$, 
  we can obtain the point 
  $\boldsymbol{p}^0$ for any point $\boldsymbol{p}^L$ on $\mathcal{M}^L$ by:
  \begin{align}
    \boldsymbol{p}^0 = \Psi(\boldsymbol{p}^L), \ \text{where} \ \Psi = (\Psi_{1, 0} \circ\cdots\circ \Psi_{L, L-1}).
  \end{align}
  The composite function $\Psi$ is also bijective, 
  hence the reverse point mapping is also supported, i.e. we have $\boldsymbol{p}^L = \Psi^{-1}(\boldsymbol{p}^0)$ 
  for any point $\boldsymbol{p}^0$ on $\mathcal{M}^0$.
  Moreover, the bijection between original mesh $\mathcal{M}^L$ and subdivided mesh $\mathcal{M}$ 
  is exactly the $\Psi$ since the subdivision does not modify the simplified geometry. 
  
  \subsection{Point-Based Feature Construction}
  \label{sec:33GeoSamp} 

\begin{figure}[!b]
  \centering
  \subfigure[]{
        \includegraphics[width=0.3\linewidth]{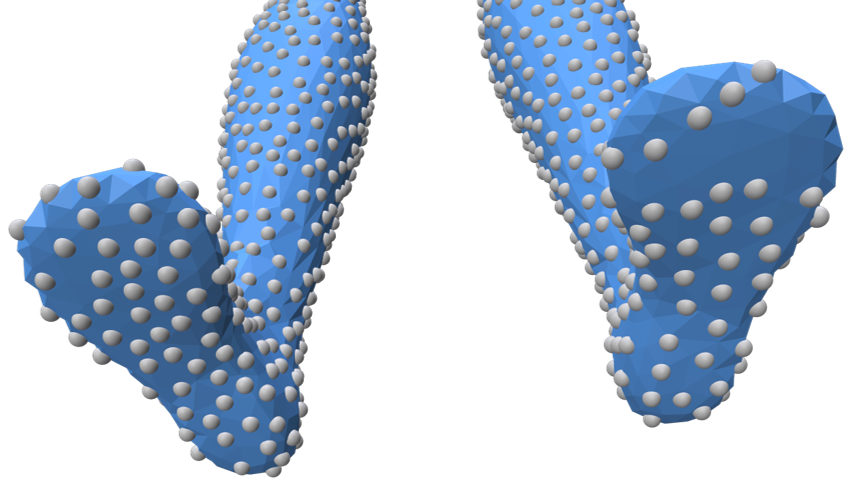}
        \label{fig:sc}
    }
  \subfigure[]{
        \includegraphics[width=0.3\linewidth]{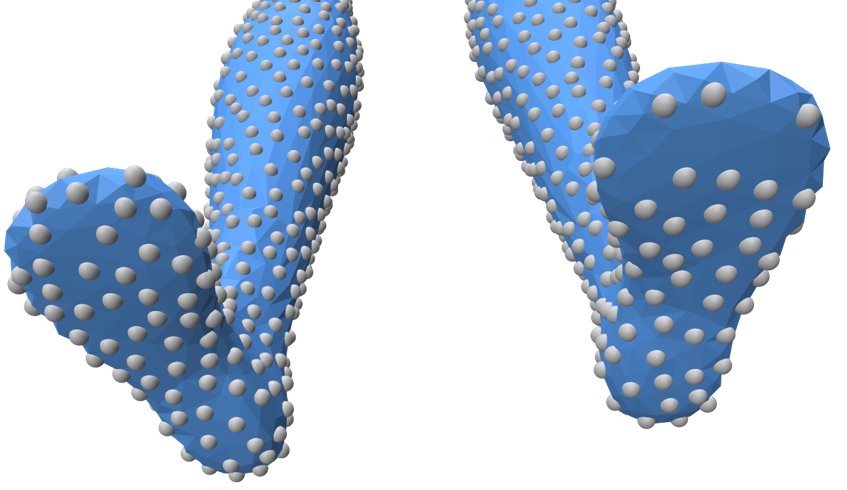}
        \label{fig:ss}
      }
  \subfigure[]{
        \includegraphics[width=0.3\linewidth]{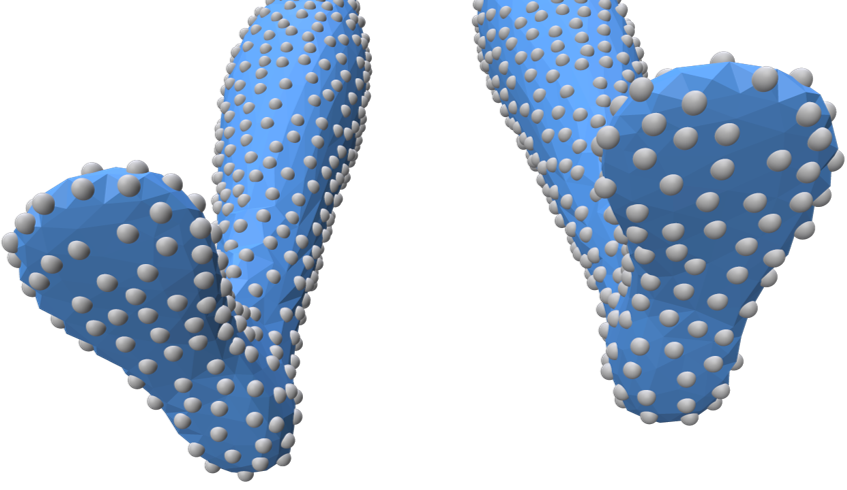}
        \label{fig:so}
      }
  \caption{Comparison among different sampling strategies, where the points are first sampled on $\mathcal{M}$ and then mapped to $\mathcal{M}^L$. (a) Barycenters of topology units. (b) Stratified sampling. (c) Distortion-aware sampling. Our strategy considers the relationship between $\mathcal{M}$ and $\mathcal{M}^L$, which results in a more uniform distribution. The example is from the HumanBody dataset~\cite{maron17humanbody}.
  } \label{fig:sampling}
\end{figure}

  Previous remesh-based methods such as SubdivNet~\cite{hu2022subdivnet} 
  and MeshMAE~\cite{liang22meshmae} typically rely on complicated descriptors, 
  including the face area, interior angles, and curvature. We argue that many of those are suboptimally hand-crafted and may 
  impede the representation learning in the subsequent neural network. 
  Moreover, there exist non-negligible geometric deviations since 
  these descriptors are built upon remeshed meshes rather than 
  the true input meshes. 
  
  In contrast, we use more simple descriptors, whereas they faithfully preserve the original geometric details achieved by the proposed point features. We approach it by first performing stratified sampling of candidate points on the subdivided mesh $\mathcal{M}$, followed by a distortion-aware selection such that the selected points approximate a \textit{uniform} distribution in the original mesh $\mathcal{M}^L$ within each topology unit. 
  

  \smallskip\noindent\textbf{Stratified Sampling.}
  For each topology unit in $\mathcal{M}$, 
  we divide it into smaller areas using recursive 1-to-4 subdivisions. 
  Note that this step is only for stratified sampling 
  rather than building topology, like in the previous section. 
  The candidate points are sampled by uniformly sampling one point within each smaller area. 
  
  We adopt the barycentric coordinate system 
  to ease the later distortion-aware selection. 
  A candidate point $\boldsymbol{p} \in \mathbb{R}^3$ can be denoted by the interpolation of the 
  triangle vertex ($\boldsymbol{v}_a, \boldsymbol{v}_b, \boldsymbol{v}_c$):
  \begin{equation}
    \boldsymbol{p} = \alpha \boldsymbol{v}_a + \beta \boldsymbol{v}_b + \gamma \boldsymbol{v}_c,
  \end{equation}
  where $\alpha, \beta, \gamma > 0$ and $\alpha + \beta + \gamma = 1$. 
  Then we can denote $\boldsymbol{p}_{bary} = (\alpha, \beta, \gamma)$ as the barycentric coordinate of point $\boldsymbol{p}$. 

  \smallskip\noindent\textbf{Distortion-Aware Selection.} 
  Incorporating a remeshed regular topology poses challenges for preserving the original geometry. The sampled points on $\mathcal{M}$ may not cover the original surface well, as shown in Fig.~\ref{fig:sampling}. To this end, we propose a novel point selection technique such that the selected points seem \textit{uniformly} sampled from the original mesh $\mathcal{M}^L$ within each topology unit. Note that uniform point sampling is a common practice in 
  point cloud neural networks~\cite{qi17pointnetv2, qian2022pointnext}.
  
  This selection is essentially computing the probability density function (PDF) 
  of points on $\mathcal{M}$ 
  given a uniform distribution on $\mathcal{M}^L$ and a mapping function. 
  Recall that $\Psi$ is a point level mapping from $\mathcal{M}^L$ to $\mathcal{M}$, 
  we have $\boldsymbol{p} = \Psi(\boldsymbol{p}^L)$, 
  and we can compute the PDF of points as 
  $\text{PDF}(\boldsymbol{p}) = \text{PDF}(\boldsymbol{p}^L) 
  |\det({\partial \boldsymbol{p}^L} / {\partial \boldsymbol{p}})|$. 
  Theoretically, given a PDF, we can apply acceptance-rejection sampling to get 
  accurate samples. 
  However, it's expensive in this case since it involves solving 
  $|\det({\partial \boldsymbol{p}^L} / {\partial \boldsymbol{p}})|$ iteratively. 

  To this end, we propose an alternative numerical method to approach it 
  based on the geometric interpretation of the mapping. 
  We employ $d S^L$ to denote the area element at $\boldsymbol{p}^L$ on $\mathcal{M}^L$. 
  Similarly, $d S$ is the area element at $\boldsymbol{p}$ on $\mathcal{M}$. 
  Then we have 
  $\text{PDF}(\boldsymbol{p}) d S = \text{PDF}(\boldsymbol{p}^L) d S^L$.  
  By replacing area elements with finite areas $\Delta S^L$ and $\Delta S$, 
  we get 
  $|\det({\partial \boldsymbol{p}^L} / {\partial \boldsymbol{p}})| \approx {\Delta S^L} / {\Delta S}$. 
  
  In practice, those finite areas are represented by the Voronoi area~\cite{meyer2003discrete}, which is efficient in computation and provides a good approximation as an area element. 
  Since the Voronoi area has no definition for points sampled on triangle faces, we interpolate the Voronoi area of triangle vertices on the candidate points using the barycentric coordinate. 
  Specifically, for a candidate point $\boldsymbol{p}$ on the mesh $\mathcal{M}$, 
  the determinant is approximated by:
  \begin{align}
    |\det({\partial \boldsymbol{p}^L} / {\partial \boldsymbol{p}})| 
    &\approx \frac{\Delta S^L}{\Delta S}  \nonumber \\ 
    &\approx \frac{{\boldsymbol{p}^L_{bary}}\cdot(\Delta S^L_1, \Delta S^L_2, \Delta S^L_3)}{{\boldsymbol{p}_{bary}}\cdot(\Delta S_1, \Delta S_2, \Delta S_3)},
  \end{align}
  where $\boldsymbol{p}_{bary}$ is the  barycentric coordinate of $\boldsymbol{p}$, 
  $\Delta S_1, \Delta S_2, \Delta S_3$ are the Voronoi areas of the triangle vertices 
  (calculated from $\mathcal{M}$). 
  The corresponding point, barycentric coordinate, 
  and Voronoi areas of vertices on the original mesh $\mathcal{M}^L$ 
  are $\boldsymbol{p}^L$, $\boldsymbol{p}^L_{bary}$ and $\Delta S^L_1, \Delta S^L_2, \Delta S^L_3$, 
  which can be calculated in a similar way. 
  
  Lastly, we weight the candidates by assigning a uniform random value 
  $w \sim U[0, |\det({\partial \boldsymbol{p}^L} / {\partial \boldsymbol{p}})|]$ 
  to each point. 
  Then the top-$k$ points are selected, where $k$ is the same for all patches. 
  The positions and normals on the original mesh $\mathcal{M}^L$ 
  of these points are packed together to represent a topology unit, 
  which is further packed to form a patch representation. 
  
  
  

\subsection{Neural Network}
\label{sec:network_arch}
We employ the standard transformer architecture~\cite{vaswani17transformer,dosovitskiy21vit}, 
which stacks multi-headed self-attention (MSA) and feed-forward network (FFN) 
to construct the neural network. 
Following MeshMAE~\cite{liang22meshmae}, we use a linear classifier for classification and train two linear heads for the segmentation task to aggregate global and local features.

\smallskip\noindent\textbf{Token and Position Embedding.} 
A patch contains 64 topology units. 
The number of selected points within each topology unit varies according 
to the specific task, detailed in experiments. 
The position and normal (on $\mathcal{M}^L$) of selected points within a patch are concatenated together, 
which are then projected into a 768-dimensional token through one linear layer. 
As for the position embedding, we employ a two-layer MLP to 
transform the average position (on $\mathcal{M}^L$) 
of selected points within a patch into a position embedding, following MeshMAE~\cite{liang22meshmae}.
Token and position embedding are added and then fed into a vanilla vision transformer. 


\section{Experiments} 
\subsection{Implementation Details}
\noindent\textbf{Mesh Processing.}
We first check and repair the meshes in the datasets to be watertight and 2-manifold so that there are no holes or boundaries on the mesh surfaces. After that, we apply the simplification described in Sec.~\ref{sec:32TopExtr} to obtain coarse meshes. Depending on the dataset, we set the number of faces in the coarse meshes between 96 and 256. 
For a small subset (such as less than 1\% for Manifold40~\cite{hu2022subdivnet}) that cannot be reduced to the given number of faces, we directly discard these meshes. 
The coarse mesh is subdivided for 3 iterations. 
Following SubdivNet~\cite{hu2022subdivnet}, we generate 10 different mesh variants in the simplification to enhance the robustness of the network.

\noindent\textbf{Data Augmentation.} 
As the datasets are small, we generate multiple inputs for each original mesh by adding randomness in simplification, we obtain different bijections between the same original mesh and the subdivided mesh (we only refer to its topology).
During training, we first resize the original mesh to a unit box, and then we apply random anisotropic scaling with a normal distribution with mean $\mu=1$ and variance $\sigma=0.1$, following MeshCNN~\cite{hanocka19meshcnn}. We further truncate the scaling factor at $3\sigma$ to avoid extreme values.
For the classification task, like MeshMAE~\cite{liang22meshmae}, we apply free-form deformation (FFD) to introduce simple deformations of rigid objects. For the HumanBody dataset~\cite{maron17humanbody}, as SubdivNet~\cite{hu2022subdivnet}, we randomly rotate the input mesh around the coordinate axes with Euler angles of $0, \pi/2, \pi,$ and $3\pi/2$ as meshes in the dataset have different orientations.

\noindent\textbf{Training Settings.} We employ ViT-Base~\cite{dosovitskiy21vit} as the backbone network in our experiments. The input channels are modified to adapt to our data. The transformer has 12  blocks with a token dimension of $768$. 

We use the AdamW~\cite{LoshchilovH19adamW} optimizer with an initial learning rate of $10^{-4}$. For classification, we train for 100 epochs and decayed the learning rate by a factor of 0.1 at 30 and 60 epochs. For segmentation, we train for 200 epochs and decay the learning rate by a factor of 0.1 at 80 and 160 epochs. 

\subsection{Classification} 
\begin{table}[!b]
      \centering
      \begin{tabular}{>{\centering\arraybackslash}p{0.33\textwidth}>{\centering\arraybackslash}p{0.09\textwidth}}
        \toprule
        Method & Acc.    \\
        \midrule
        PointNet++~\cite{qi17pointnetv2} & 87.9\% \\
        PointNeXt-S~\cite{qian2022pointnext}&  92.3\% \\
        \midrule
        MeshNet~\cite{feng19meshnet} & 88.4\% \\
        SubdivNet~\cite{hu2022subdivnet} & 91.2\% \\
        MeshMAE~\cite{liang22meshmae} & 91.7\%\\
        \midrule
        Ours     & \textbf{92.5\%}  \\
        \bottomrule
      \end{tabular}
\caption{Classification results on Manifold40~\cite{hu2022subdivnet}. The first two methods take point clouds as input. Other methods are mesh-based methods.}
      \label{tab:classification}
\end{table}
We adopt the Manifold40 dataset~\cite{hu2022subdivnet} for the classification task. The Manifold40 dataset is 
converted from ModelNet40~\cite{wu15modelnet} to guarantee that meshes are watertight and 2-manifold. It contains 12311 shapes in 40 categories, such as airplanes, cars, and plants. We specify the minimum face number as 96 and the maximum face number as 256 during mesh simplification. Less than 1\% of shapes are discarded as they cannot reach the given face numbers. During training and testing, we set the patch number as 256 and fill the empty patches with zeros.
For the classification task, we select one point from the 64 sampled candidate points within each topology unit.

To demonstrate the capabilities of our method, we compared it with several mesh-based and point cloud-based methods. The results are shown in Table~\ref{tab:classification}. For a fair comparison, the results of point cloud methods are also evaluated on the Manifold40 dataset~\cite{hu2022subdivnet}. This dataset is more challenging than ModelNet40~\cite{wu15modelnet}, resulting in decreased accuracy.

\subsection{Segmentation}
\begin{figure}[!tb]
  \hspace{12pt}
  \stackunder[5pt]{
        \includegraphics[width=0.18\linewidth]{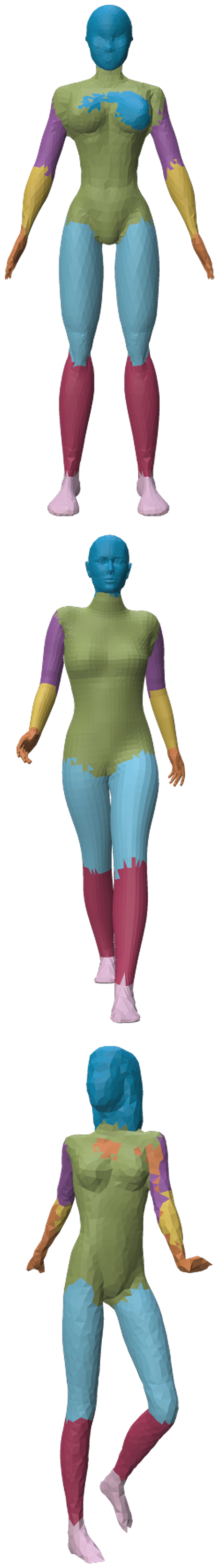}
      }{SubdivNet}
  \stackunder[5pt]{
        \includegraphics[width=0.18\linewidth]{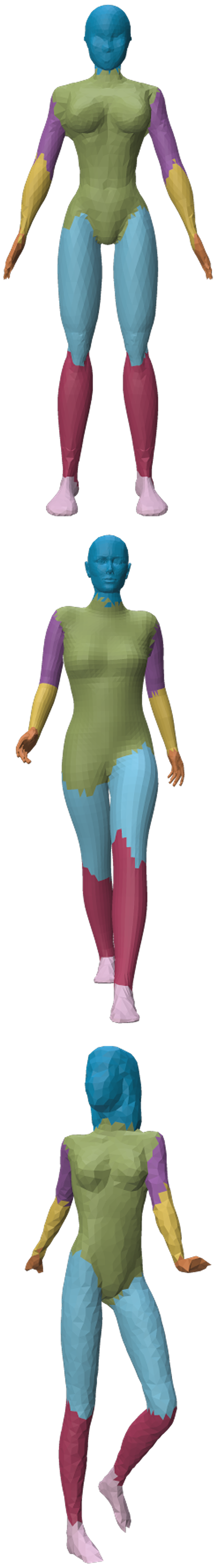}
      }{MeshMAE}
  \stackunder[5pt]{
        \includegraphics[width=0.18\linewidth]{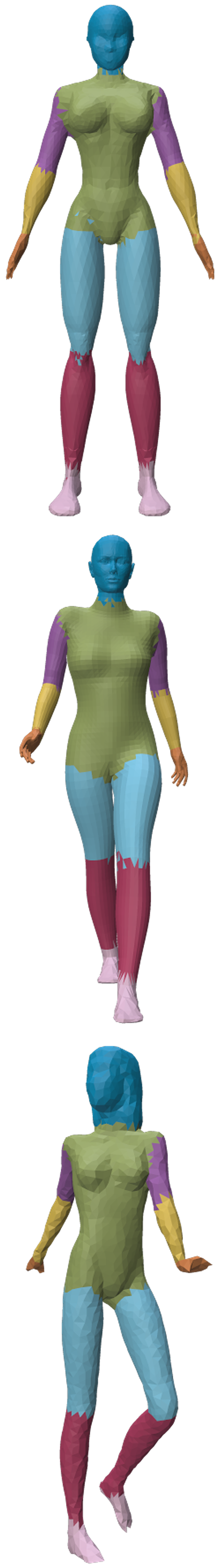}
      }{Ours}
  \stackunder[5pt]{
        \includegraphics[width=0.18\linewidth]{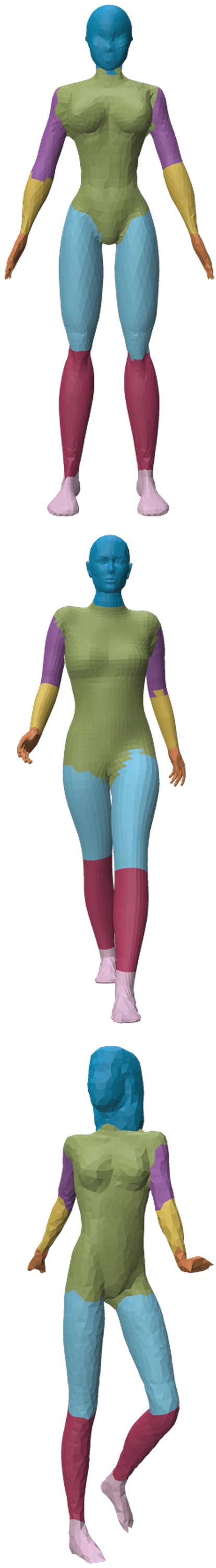}
      }{Ground-Truth}
  \caption{\label{fig:seg_result}Segmentation results on the HumanBody
dataset~\cite{maron17humanbody}. Segmentation parts are distinguished by different colors. Our method gives more accurate results.}
\end{figure}

For the segmentation task, our neural mesh network predicts labels for each topology unit on the processed input (simplified and then subdivided) because of the utilization of the regular topology. Consequently, we need to give appropriate labels for topology units and map the predicted labels back to the original mesh. 
The mapping between the original mesh and the input one is built through the point-to-point bijection mentioned in Sec.~\ref{sec:32TopExtr}. 
The label of each topology unit is voted by the labels of selected points, while original mesh segmentation is evaluated by labeling its faces according to their centroids' topology unit.
We adopt two datasets, HumanBody~\cite{maron17humanbody} and COSEG-aliens~\cite{wang12coseg}, for the segmentation task.


\begin{table}[t]
          \centering
          \begin{tabular}{{p{0.025\textwidth}>{\centering\arraybackslash}p{0.28\textwidth}>{\centering\arraybackslash}p{0.09\textwidth}}}
          \addlinespace[3pt]
            \clineB{2-3}{2}
            \addlinespace[3pt]
            &Method     & Acc.    \\
            \addlinespace[3pt]
            \cline{2-3}
            \addlinespace[6pt]
            \ldelim\{{11}{*}[$\dagger$]
            &GCNN~\cite{masci15gcnn} & 85.4\%\\
            &PointNet~\cite{qi17pointnet} & 74.7\%\\
            &PointNet++~\cite{qi17pointnetv2} & 82.3\%\\
            &MeshCNN~\cite{hanocka19meshcnn} & 87.8\%\\
            &DGCNN~\cite{wang19dgcnn} & 87.8\%\\
            &CGConv~\cite{yang21cgconv} & 89.9\%\\
            &PD-MeshNet~\cite{milano20pdmeshnet} & 86.9\%\\
            &DiffusionNet~\cite{sharp22diffnet} & 91.7\%\\
            &SubdivNet~\cite{hu2022subdivnet} & 93.0\%\\
            &MeshMAE~\cite{liang22meshmae} & 90.0\%\\
            &Ours & \textbf{93.2\%}\\
            \addlinespace[3pt]
            \cline{2-3}
            \addlinespace[6pt]
            \ldelim\{{7}{*}[$\ddagger$] & MeshCNN~\cite{hanocka19meshcnn} & 85.4\%\\
            &PD-MeshNet~\cite{milano20pdmeshnet} & 85.6\% \\
            &HodgeNet~\cite{smirnov21hodgenet} & 85.0\% \\
            &DiffusionNet~\cite{sharp22diffnet} & 90.8\%\\
            &SubdivNet~\cite{hu2022subdivnet} & 90.8\% \\
            &MeshMAE~\cite{liang22meshmae} & 90.1\%\\
            &Ours     &  \textbf{91.1\%}   \\            
            \addlinespace[3pt]
            \clineB{2-3}{2}
            \addlinespace[3pt]
          \end{tabular}
            \caption{Segmentation results on the HumanBody dataset~\cite{maron17humanbody}. 
            The $\dagger$ rows are evaluations on the original meshes, and the $\ddagger$ rows are evaluations on the processed inputs.
            }
          \label{tab:seghuman}
\end{table}

\begin{table}[t]
          \centering
          \begin{tabular}{>{\centering\arraybackslash}p{0.33\textwidth}>{\centering\arraybackslash}p{0.09\textwidth}}
            \toprule
            Method     & Acc.    \\
            \midrule
            PD-MeshNet~\cite{milano20pdmeshnet} & 89.0\% \\
            MeshCNN~\cite{hanocka19meshcnn} & 94.4\% \\
            SubdivNet~\cite{hu2022subdivnet} & 97.3\% \\
            MeshMAE~\cite{liang22meshmae} & \textbf{97.9\%} \\
            \midrule
            Ours     &  \textbf{97.9\%}   \\
            \bottomrule
          \end{tabular}
\caption{Segmentation results on COSEG-aliens dataset~\cite{wang12coseg}.}
          \label{tab:segalien}
\end{table}

\begin{table}[!t]
          \centering
          \begin{tabular}{>{\centering\arraybackslash}p{0.155\textwidth}>{\centering\arraybackslash}p{0.155\textwidth}>{\centering\arraybackslash}p{0.09\textwidth}}
            \toprule
            Bij. Variants  & Geo. Selected  & Acc.   \\
            \midrule
            10 & 16 & \textbf{91.1\%} \\
            10 & 4 & 91.0\% \\
            10 & 1 & 90.9\% \\
            10 & 1* & 90.6\% \\
            3 & 16 & 90.6\% \\
            1 & 16 & 89.6\% \\
            \bottomrule
          \end{tabular}
\caption{Ablation studies on the impact of variants in mesh simplification and sampling strategy on the HumanBody dataset~\cite{maron17humanbody}. The row with * represents an unweighted geometry sampling, which gives a worse result compared to the weighted one. The first row is our default setting.}
          \label{tab:ablation1}
\end{table}

\smallskip\noindent\textbf{HumanBody.} The HumanBody dataset~\cite{maron17humanbody} contains 381 training shapes and 18 test shapes. The training set consists of four sub-datasets, SCAPE~\cite{anguelov05scape}, FAUST~\cite{bogo14faust}, MIT~\cite{vlasic08mit}, and Adobe Fuse~\shortcite{adobe2021adobe}. The test set uses shapes from SHREC07~\cite{giorgi07shrec}. Every human body in the dataset is divided into eight segments. In this experiment, we select 16 points from the 64 sampled candidate points within each topology unit. 
The results are shown in Table~\ref{tab:seghuman}. 
Examples of the segmentation results are visualized in Fig.~\ref{fig:seg_result}. 
SubdivNet~\cite{hu2022subdivnet} displays intersecting boundaries in upper body segments, and MeshMAE~\cite{liang22meshmae} exhibits an overall offset in results (especially in lower bodies), leading to decreased accuracy. Our method is better in both boundary and accuracy.

\smallskip\noindent\textbf{COSEG-aliens.} The COSEG-aliens dataset~\cite{wang12coseg} contains 200 shapes, where each mesh is divided into 4 segments. Following SubdivNet~\cite{hu2022subdivnet}, we split the training and test set in a 4:1 ratio. 
In this experiment, we set the selected point number as 16 from 64 points for each face. 
The results are shown in Table~\ref{tab:segalien}.

\subsection{Ablation Studies}
In this section, the important components in SieveNet, including the detailed settings in our two stages
and the feature descriptor, are studied.

\begin{table}[t]
    \centering
    \begin{tabular}{>{\centering\arraybackslash}p{0.25\textwidth}>{\centering\arraybackslash}p{0.17\textwidth}}
    \toprule
    Feature Descriptor & Acc.  \\
    \midrule
    Position & 88.4\% \\
    Position + Normal &  \textbf{91.1\%} \\
    \bottomrule
    \end{tabular}
    \caption{Ablation studies on the impact of the feature descriptor, based on the HumanBody dataset~\cite{maron17humanbody}}
    \label{tab:ablation2}
\end{table}
\smallskip\noindent\textbf{Variants in Topology and Geometry.}
In this section, we discuss the influence of settings in Stage 1 (Sec.~\ref{sec:32TopExtr}) and Stage 2 (Sec.~\ref{sec:33GeoSamp}).
In Stage 1, due to the randomness in simplification, for each original mesh $\mathcal{M}^L$, we extract a subdivided topology $\mathcal{M}$ with different bijections. The variants of bijections are denoted as \textit{Bij. Variants} in Table~\ref{tab:ablation1}. Increasing variants augment the small dataset, and make the network insensitive to abnormal bijections as well, which results in better accuracy.
In stage 2, the number of candidate points in each topology unit is fixed as 64. The number of selected points is denoted as \textit{Geo. Selected} in  Table~\ref{tab:ablation1}. Thanks to the distortion-aware selection, decreasing selecting numbers only results in a negligible loss in accuracy. On the other side, an unweighted uniform sampling leads to a worse result (90.9\% $\rightarrow$ 90.6\%).

\smallskip\noindent\textbf{Utilized Features.} As mentioned in Sec.~\ref{sec:33GeoSamp}, the selected points are utilized to capture features on the original mesh. A point can be completely described by its position and normal,  so we use both of them for our network. As shown in Table~\ref{tab:ablation2}, adding the normal information significantly improves the performance.

\smallskip\noindent\textbf{Point Order.} 
In our scheme, the topology units within a patch are arranged in fixed and regular order. However, the selected points within a topology unit are unordered. 
As a result, the feature representation of a topology unit varies according to the point order. We thus study whether the order of points matters. 
We use two strategies to verify the impact of point order. Firstly, we randomly permute the point order and evaluate the performance on HumanBody~\cite{maron17humanbody}. Secondly, we introduce a standalone transformer block to points (inside the topology unit) without position embedding to perform unordered feature points fusion. Interestingly, we find both of these achieve similar results with our original simple design, suggesting a robust nature of utilizing a highly regular topology of the topology unit.

\section{Conclusions}
In this paper, we propose SieveNet to learn from meshes by combining structured topology and accurate geometric features. Experiments on several classification and segmentation datasets have demonstrated the superiority of our method. In addition, SieveNet eliminates the need for hand-crafted features and can leverage the off-the-shelf transformer architecture. 
In the future, we will investigate how to assign variable numbers of points for each topology unit, based on the information density of the original geometry.

\bibliography{aaai24}

\appendix
\section{Proof}
In this section, we provide a comprehensive proof of the formulas introduced in \textit{Distortion-Aware Selection} (Sec. 3.2). Initially, we will establish an analytical proof demonstrating that $\text{PDF}(\boldsymbol{p}) = \text{PDF}(\boldsymbol{p}^L) \cdot |\det({\partial \boldsymbol{p}^L} / {\partial \boldsymbol{p}})|$. Subsequently, we will offer a geometric perspective, clarifying that $|\det({\partial \boldsymbol{p}^L} / {\partial \boldsymbol{p}})| = {d S^L}/{d S}$. Lastly, we will reach the approximate formula used in practice.

To maintain conciseness, we present the proof in the context of the two-dimensional scenario, with easy generalization to higher dimensions. We also employ slightly altered notation compared to the paper to enhance the proof's brevity. Specifically, $\text{PDF}(\boldsymbol{p}^L)$ is denoted as $f_{X, Y}$, and $\text{PDF}(\boldsymbol{p})$ is denoted as $f_{U, V}$.

\begin{theorem}
Let $X=X(U, V)\in \mathbb{R}$, $Y=Y(U, V)\in \mathbb{R}$, with joint density function $f_{X, Y}$. Assuming a bijection exists, thus $U=U(X,Y)$, $V=V(X,Y)$. For points $P^L=(X,Y)\in \mathbb{R}^2$, $P=(U,V)\in \mathbb{R}^2$. Define the joint function $\Psi:\mathbb{R}^2\rightarrow\mathbb{R}^2$ by
\begin{equation}
    \Psi(\boldsymbol{p}^L(x, y)) = (u(x,y), v(x,y)).
\end{equation}
Then the joint density function of $U$ and $V$ is given by
\begin{equation}
\label{eq1}
    f_{U, V} = f_{X, Y}(\Psi^{-1}(u, v))|J(\Psi^{-1}(u, v))|,
\end{equation}
where $J=\det({\partial \boldsymbol{p}^L} / {\partial \boldsymbol{p}})$ is the Jacobian derivative of $\Psi$ and where $\Psi^{-1}(u, v)$ is the inverse function of $\Psi(x, y)$.
\end{theorem}

\begin{proof}
For $a\leq b$ and $c\leq d$, it holds that
\begin{equation}
    F(a\leq U\leq b, c\leq V\leq d) = \int_c^d\int_a^bf_{U,V}(u, v) d u d v,
\end{equation}
where $F$ is a cumulative distribution function. Consider the two-dimensional rectangle $S=[a\ b]\times[c\ d]$. This can be rewritten as
\begin{equation}
    F(\boldsymbol{p}(u, v)\in S) = \iint_{S}f_{U,V}(u, v) d u d v.
\end{equation}
Recalling that $(x, y)=\Psi^{-1}(u, v)$, and applying the integral transformation, we have
\begin{align}
    &\iint_{S}f_{U,V}(u, v) d u d v\notag\\
    =&\iint_{S}(f_{X,Y}\Psi^{-1}(u, v)|J(\Psi^{-1}(u, v))|) d u d v\notag\\
    =&\iint_{\Psi^{-1}(S)}(f_{X,Y}(x,y)|J(x,y)|)/|J(x,y)| d x d y\notag\\
    =&\iint_{\Psi^{-1}(S)} f_{X,Y}(x,y) dxdy\notag\\
    =&F_{X,Y}((x,y)\in \Psi^{-1}(S))\notag\\
    =&F_{X,Y}(\Psi^{-1}(u, v)\in \Psi^{-1}(S))\notag\\
    =&F_{U,V}(\boldsymbol{p}(u, v)\in S),\label{eq2}
\end{align}
as required.
\end{proof}
The result can be extended to higher dimensions with the Jacobian technique.

\medskip\noindent In Eq.~(\ref{eq2}), we observe that
\begin{equation}
F_{X,Y}(\Psi^{-1}(u, v)\in \Psi^{-1}(S))
=F_{U,V}(\boldsymbol{p}(u, v)\in S),
\end{equation}
or, using clearer notation:
\begin{equation}
F_{X,Y}(\boldsymbol{p}^L(x,y)\in S^L)
=F_{U,V}(\boldsymbol{p}(u, v)\in S).
\end{equation}
By differentiating both sides, we obtain
\begin{equation}
f_{X,Y}(\boldsymbol{p}^L(x,y)) d S^L
=f_{U,V}(\boldsymbol{p}(u, v)) d S.\label{eq3}
\end{equation}
On the other hand, from Eq.~(\ref{eq1}), we know that
\begin{equation}
f_{U, V}(\boldsymbol{p}(u, v)) = f_{X, Y}(\boldsymbol{p}^L(x, y))|J(\boldsymbol{p}^L(x, y))|.\label{eq4}
\end{equation}
Comparing Eq.~(\ref{eq3}) and Eq.~(\ref{eq4}), we conclude that $|J| = {d S^L}/{d S}$. From the geometric aspect,
the absolute value of the Jacobian $|J| = |\det({\partial \boldsymbol{p}^L} / {\partial \boldsymbol{p}})|$, represents the distortion factor for area when transitioning from $u-v$ coordinates to $x-y$ coordinates. 
Hence, it is both intuitive and reasonable to substitute $|\det({\partial \boldsymbol{p}^L} / {\partial \boldsymbol{p}})|$ with ${d S^L}/{d S}$.

In practice, we utilize finite areas denoted as $\Delta S^L$ and $\Delta S$, associated with Voronoi cells, as heuristic surrogates for the ratio of area elements. A Voronoi cell consists of all points that are closer to its central vertex $v$ than to any other, as illustrated in Fig.~\ref{fig1}. Notably, Voronoi areas offer a reliable approximation for area elements~\cite{meyer2003discrete}.
\begin{figure}
    \centering
    \includegraphics[width=0.4\linewidth]{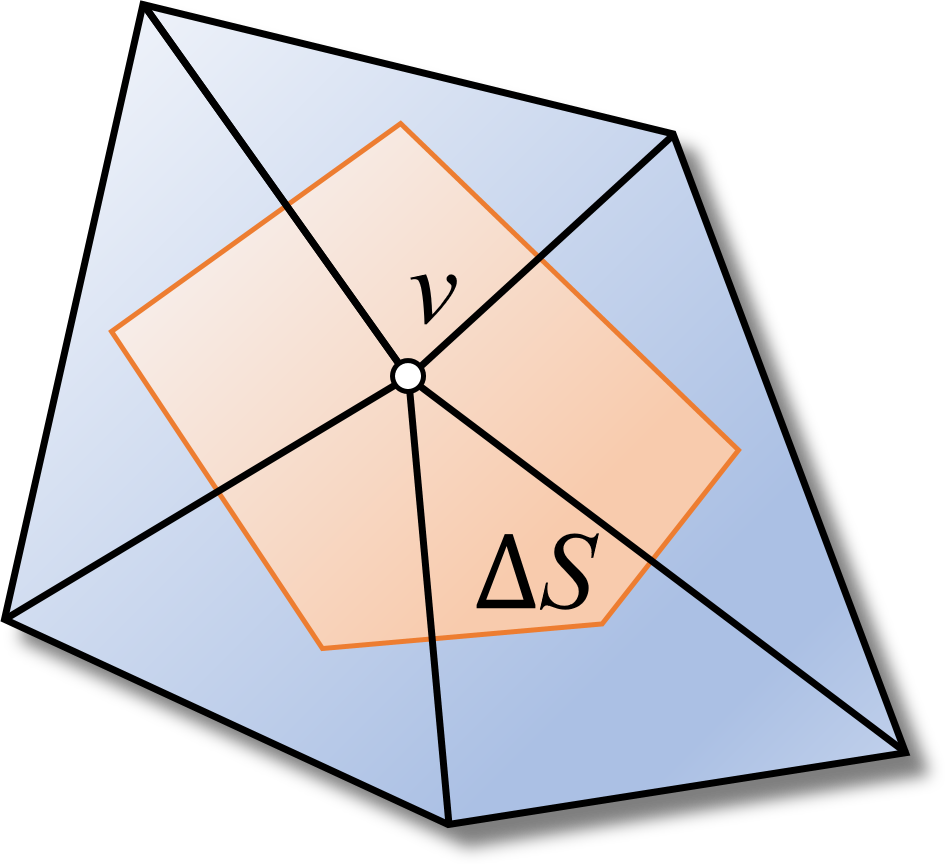}
    \caption{The finite area $\Delta S$ on a triangulated one-ring surface using Voronoi cells.}
    \label{fig1}
\end{figure}
As the Voronoi area lacks a consistent definition for points sampled on triangle faces, we interpolate the Voronoi area of triangle vertices on the candidate points using barycentric coordinates, yielding the following relation:
\begin{equation}
    \frac{d S^L}{d S} \approx 
    \frac{\Delta S^L}{\Delta S} \approx 
    \frac{{\boldsymbol{p}^L_{bary}}\cdot(\Delta S^L_1, \Delta S^L_2, \Delta S^L_3)}{{\boldsymbol{p}_{bary}}\cdot(\Delta S_1, \Delta S_2, \Delta S_3)},
\end{equation}
which is used as the weight of distortion-awared selection in our paper.

\section{Discussions on Topology Extraction}
\subsection{Comparison with Remeshing}
The crucial difference between our method and the remeshing technique lies in our treatment of the interplay between mesh topology and geometry. The conventional concept of remeshing primarily aims to produce a new mesh with limited geometric errors and a desired regular topology. However, this inadvertently restricts the geometric representation by enforcing strict adherence to the vertex-face structure. This restriction could potentially lead to a bottleneck in the transmission of geometric information.

In contrast, our method embraces a more flexible perspective. We relax the constraint mentioned earlier, recognizing that the primary emphasis lies in achieving a regular topology, rather than rigidly following the vertex-face structure. This conscious relaxation of the requirement empowers us to exploit a wider array of information from the original mesh. To implement this concept, we employ an approach that directly combines the attributes of points on the original mesh, thereby facilitating the integration of features into the topology units. It's important to emphasize that our approach has the capacity to accommodate more intricate features, a direction we intend to explore in our future work.

In summary, the fundamental distinction between our proposed method and the conventional remeshing technique rests in our adaptable approach to mesh topology and geometry interaction. By prioritizing a regular topology within topology units over strict adherence to the conventional vertex-face structure, we not only achieve the goal of geometric representation but also hold the potential to incorporate more intricate features in future advancements of our approach.

\subsection{Alternative Implementations}
We explored implementing our topology extraction algorithm adapted from different remeshing algorithms. However, the available implementations differ significantly in computational efficiency. We first instantiated our algorithm from MAPS~\cite{lee98maps}, while the pure Python implementation within MAPS is inefficient. Due to resource limitations on our CPUs, we turn to a C++-based remeshing implementation~\cite{liu20neusubdiv} for all of our experiments.


\section{Experiements}
\subsection{Computing Infrastructure}
For all of our experiments, the network is trained on 2 NVIDIA TITAN Xp GPUs, with an Intel(R) Core(TM) i7-6700k @ 4.0GHz CPU.

\subsection{Further Ablation Study Details}
To discuss whether the unordered point-based features inside a topology unit affect the performance of our network, we explore either randomly permuting the input inside topology units or using an extra transformer block to extract permutation-invariant features. 

For the random permutation approach, we employ the numpy library to shuffle the order in the training phase, independently shuffling the point order within each topology unit. Alternatively, in the transformer approach, we attempt a shared transformer block, where a conventional transformer block is deployed to embed the ordered features in each topology unit into a point-order invariant one. The results are shown in Table~\ref{tab1}. The default method achieves the best accuracy, indicating the effectiveness and robustness of our method.

\begin{table}[H]
    \centering
    \begin{tabular}{>{\centering\arraybackslash}p{0.25\textwidth}>{\centering\arraybackslash}p{0.17\textwidth}}
    \toprule
    Input Processing & Acc.  \\
    \midrule
    N/A & \textbf{91.1\%} \\
    Permutation & 90.8\%\\
    Transformer & 91.0\%\\
    \bottomrule
    \end{tabular}
    \caption{Impact of point order on the segmentation performance on the HumanBody dataset~\cite{maron17humanbody}.}
    \label{tab1}
\end{table}

\end{document}